\documentclass{article}
\pdfoutput=1

\usepackage{amsmath,amsfonts,bm}

\def\eqref#1{equation~\ref{#1}}
\def\Eqref#1{Equation~(\ref{#1})}

\def\1{\bm{1}}
\newcommand{\train}{\mathcal{D}}
\newcommand{\valid}{\mathcal{D_{\mathrm{valid}}}}

\def\va{{\bm{a}}}
\def\vb{{\bm{b}}}

\def\vh{{\bm{h}}}

\def\vp{{\bm{p}}}
\def\vq{{\bm{q}}}

\def\vu{{\bm{u}}}

\def\vw{{\bm{w}}}
\def\vx{{\bm{x}}}
\def\vy{{\bm{y}}}
\def\vz{{\bm{z}}}
\def\vrho{{\bm{\rho}}}

\def\eva{{a}}
\def\evb{{b}}

\def\evp{{p}}
\def\evq{{q}}

\def\evy{{y}}
\def\evz{{z}}

\def\vpt{\tilde{\bm{p}}}
\def\vqt{\tilde{\bm{q}}}
\def\evqt#1{{\tilde{q}_{#1}}}
\def\evpt#1{{\tilde{p}_{#1}}}
\def\evct#1{{\tilde{c}_{#1}}}

\def\mW{{\bm{W}}}

\DeclareMathAlphabet{\mathsfit}{\encodingdefault}{\sfdefault}{m}{sl}
\SetMathAlphabet{\mathsfit}{bold}{\encodingdefault}{\sfdefault}{bx}{n}

\def\gH{{\mathcal{H}}}

\def\gL{{\mathcal{L}}}

\def\gX{{\mathcal{X}}}
\def\gY{{\mathcal{Y}}}

\def\sR{{\mathbb{R}}}
\def\sS{{\mathbb{S}}}

\def\sZ{{\mathbb{Z}}}

\newcommand{\E}{\mathbb{E}}

\newcommand{\softmax}{\mathrm{softmax}}

\newcommand{\LS}{\mathrm{LS}}
\newcommand{\KD}{\mathrm{KD}}
\newcommand{\kdpt}{\mathrm{KD}\text{-pt}}
\newcommand{\kdrel}{\mathrm{KD}\text{-rel}}
\newcommand{\kdsim}{\mathrm{KD}\text{-sim}}
\newcommand{\kdtopk}{\mathrm{KD}\text{-topk}}

\newcommand{\kdptsim}{\mathrm{KD}\text{-pt+sim}}

\DeclareMathOperator*{\argmax}{arg\,max}
\DeclareMathOperator*{\argmin}{arg\,min}

\usepackage[preprint]{neurips_2020}

\usepackage[utf8]{inputenc} %
\usepackage[T1]{fontenc}    %
\usepackage{hyperref}       %
\hypersetup{colorlinks,     %
linkcolor=black,
citecolor=blue
}
\usepackage{xr-hyper}
\usepackage{url}            %
\usepackage{booktabs}       %
\usepackage{amsfonts}       %
\usepackage{nicefrac}       %
\usepackage{microtype}      %

\usepackage{amssymb}
\usepackage{amsthm}
\usepackage{amsmath}
\usepackage{cleveref}
\usepackage{thmtools}
\usepackage{thm-restate}
\usepackage{graphicx}
\usepackage{caption}
\usepackage{subcaption}
\usepackage{wrapfig}
\usepackage{multirow}
\makeatletter
\newcommand*{\addFileDependency}[1]{%
  \typeout{(#1)}
  \@addtofilelist{#1}
  \IfFileExists{#1}{}{\typeout{No file #1.}}
}
\makeatother

\usepackage{xcolor}

\newtheorem{thm}{Theorem}

\newtheorem{prop}[thm]{Proposition}

\title{Understanding and Improving Knowledge Distillation}

\author{%
  Jiaxi Tang\thanks{Work done when was a student at Simon Fraser University and while interning at Google.}, Rakesh Shivanna, Zhe Zhao, Dong Lin, Anima Singh, Ed H.Chi, Sagar Jain \\
  Google, Inc \\
  \texttt{\{jiaxit,rakeshshivanna,zhezhao,dongl,animasingh,edchi,sagarj\}@google.com} \\
}

\begin{document}

\maketitle

\begin{abstract}
Knowledge Distillation (KD) is a model-agnostic technique to improve model quality while having a fixed capacity budget. It is a commonly used technique for model compression, where a larger capacity teacher model with better quality is used to train a more compact student model with better inference efficiency. Through distillation, one hopes to benefit from student's compactness, without sacrificing too much on model quality.
Despite the large success of knowledge distillation, better understanding of how it benefits student model's training dynamics remains under-explored.
In this paper, we categorize teacher's knowledge into three hierarchical levels and study its effects on knowledge distillation: (1) knowledge of the `universe', where KD brings a regularization effect through label smoothing; (2) domain knowledge, where teacher injects class relationships prior to student's logit layer geometry; and (3) instance specific knowledge, where
teacher rescales student model's per-instance gradients based on its measurement on the event difficulty.
Using systematic analyses and extensive empirical studies on both synthetic and real-world datasets, we confirm that the aforementioned three factors play a major role in knowledge distillation.
Furthermore, based on our findings, we diagnose some of the failure cases of applying KD from recent studies.

\end{abstract}

\section{Introduction} Recent advances in artificial intelligence have largely been driven by learning deep neural networks, and thus, current state-of-the-art models typically require a high inference cost in computation and memory.
Therefore, several works have been devoted to find a better quality and computation trade-off, such as pruning~\citep{han2015learning} and quantization~\citep{han2015deep, jacob2018quantization}.
One promising and commonly used method for addressing this computational burden is Knowledge Distillation (KD), proposed by~\citet{hinton2015distilling}, which uses a larger capacity teacher model (ensembles) to transfer its `dark knowledge' to a more compact student model. Through distillation, one hopes to achieve a student model that not only inherits better quality from the teacher, but is also more efficient for inference due to its compactness. 
Recently, we have witnessed a huge success of knowledge distillation, irrespective of the model architecture and application domain~\citep{kim2016sequence, chen2017darkrank, tang2018ranking, Anil2018Large, he2019bag}.

Despite the large success of KD, surprisingly sparse research has been done to better understand the mechanism of how it works, which could limit the applications of KD and also raise unexpected or unexplainable results. For example, to successfully `distill' a better student, one common practice is to have a teacher model with as good quality as possible. However, recently~\citet{mirzadeh2019improved} and~\citet{muller2019does} have found this intuition would fail under certain circumstances. Furthermore,~\citet{Anil2018Large} and~\citet{furlanello2018born} have analyzed that even without using a powerful teacher, distilling a student model to itself using mutual or self-distillation also improves quality.
To this end, some researchers have made attempts on understanding the mechanism of KD.
For example,~\citet{yuan2019revisit} connects label smoothing to KD. \citet{furlanello2018born} conjectures KD's effect on re-weighting training examples. In this work, we found that the benefits of KD comes from a combination of multiple effects, and propose partial KD methods to dissect each of the effects.

This work is an attempt to shed light upon the `dark knowledge' distillation, making this technique less mysterious. 
More specifically, we make the following contributions:
\begin{itemize}
\itemsep0em
\item For KD on multi-class classification task, we systematically break down its effects into: (1) label smoothing from universal knowledge, (3) injecting domain knowledge of class relationships to student's output logit layer geometry, and (2) gradient rescaling based on teacher's measurement of instance difficulty. We provide theoretical analyses on how KD exhibits these effects, and improves student model's quality (Section~\ref{sec:analysis}).
\item We propose partial-distillation techniques using hand-crafted teacher's output distribution (Section~\ref{sec:pseudo_kd}) to simulate and validate different effects of knowledge distillation.
\item We empirically demonstrate and confirm our hypothesis on the effects of KD on both synthetic and real-world datasets.
Furthermore, using our understanding, we diagnose some recent failures of applying KD~(Section~\ref{sec:application}).
\end{itemize}

\section{Related Work} In the context of deep learning, knowledge transfer has been successfully used to effectively compress the power of a larger capacity model (a teacher) to a smaller neural network (a student). Adopting this teacher-student learning paradigm, many forms of \emph{knowledge} have been investigated: layer activations~\citep{romero2014fitnets}, auxiliary information~\citep{vapnik2015learning}, Jacobian matrix of the model parameters~\citep{czarnecki2017sobolev, srinivas2018knowledge}, Gram matrix derived from pairs of layers~\citep{yim2017gift}, activation boundary~\citep{heo2019knowledge}, etc. 
Among these, the \emph{original (or vanilla)} KD -- learning from teacher's output distribution~\citep{hinton2015distilling} is the most popular.
Besides compression, $\KD$~has also been successfully applied to improve generalization~\citep{furlanello2018born}, reproducibility~\citep{Anil2018Large}, defend adversarial attacks~\citep{papernot2016distillation}, etc. 

Though $\KD$ has been successfully applied in various domains, there has been very few attempts on understanding how and why it helps neural network training. \citet{hinton2015distilling} argued that the success of $\KD$ could be attributed to the output distribution of the incorrect classes, which provides information on class relationships. From learning theory perspective,~\citet{vapnik2015learning} studied the effectiveness of knowledge transfer using auxiliary information, known as \emph{Privileged Information}. Following which,~\citet{lopez2015unifying} established the connection between $\KD$~and privileged information. Recently, \citet{phuong2019towards} showed a faster convergence rate from distillation. However, most of the existing theoretical results rely on strong assumptions (e.g., linear model, or discarding ground-truth when training the student), and also fails to explain %
the recent failure cases of distilling from a better quality teacher~\citep{mirzadeh2019improved, muller2019does}.

The most relevant work to our own is \citep{furlanello2018born}. Though the main focus of their work is to propose $\KD$ techniques to boost quality, they also provide intuitions for the effectiveness of $\KD$. In our work, we offer theoretical analysis on some of their conjectures, and improve on erroneous assumptions. Furthermore, we systematically investigate the mechanism behind knowledge distillation by decomposing its effects, and analyzing how each of these effects helps with student model's training using our proposed partial-distillation methods.

\section{Analyzing Mechanisms of Knowledge Distillation} \label{sec:analysis} In this section, we provide a systematic analyses for the mechanisms behind $\KD$ based on theoretical and empirical results. We start by introducing essential background, dissect distillation benefits from three main effects, and conclude by connecting and summarizing these effects.

\textbf{Background.}
Consider the task of classification over $[K] := \{1\ldots K\}$ classes, given $(\vx, \vy) \in \gX\times\gY$, with $\vy\in\{0,1\}^K$ denoting one-hot encoded label, and $t \in [K]$ denoting the ground-truth class. The goal is to learn a parametric mapping function $f(\vx; \theta) : \gX \mapsto \gY$ where $\theta \in \Theta$  can be characterized by a neural network. We learn the parameters $\theta$ via Empirical Risk Minimization of the surrogate loss function, typically optimized using a variant of Stochastic Gradient Descent:
$ \theta^* = \argmin_{\theta \in \Theta} \gL(\vy, f(\vx; \theta)),$
where $\gL$ is the cross-entropy loss $\gH(\vy, \vq)=\sum_{i=1}^{K} -\evy_i \log \evq_i$, and $\vq=f(\vx; \theta)$ is the network's output distribution computed by applying $\softmax$ over the logits $\vz$: 
$    \evq_i = \softmax(\evz_i) = \frac{\exp(\evz_i)}{\sum_{j=1}^K \exp(\evz_j)}$. 
We could also scale the logits by temperature $T>1$ to get a smoother distribution $\evqt{i} = \softmax(\evz_i/T)$. Gradient of a single-sample \emph{w.r.t.} logit $\evz_i$ is given by:
\begin{equation}
    \label{eq:grad}
    \partial\gL / \partial \evz_i = \evq_i - \evy_i. \quad\text{Lets denote }\partial_i = \partial\gL / \partial \evz_i.
\end{equation}
\subsection{Knowledge of the universe -- benefits from label smoothing}

Label Smoothing (LS)~\citep{szegedy2016rethinking} is a technique to soften one-hot label $\vy$ by a factor of $\epsilon$, such that the modified label becomes:  $\tilde{\evy}_i^{\LS}=(1-\epsilon)\evy_i + \epsilon/K$. Label smoothing mitigates the over-confidence issue of neural networks, and improves model calibration~\citep{muller2019does}.
Knowledge Distillation~($\KD$) on the other hand, uses an additional teacher model's predictions $\vp$ for training:
$$\theta^*_{\KD} =\argmin_{\theta \in \Theta} \big\{ \gL^{\KD}(\vy, \vp, f(\vx; \theta), \lambda, T) =(1-\lambda)\gH(\vy, \vq) + \lambda \gH(\vpt, \vqt)\big\},$$
where $\lambda\in[0,1]$ is a hyper-parameter; and $\vqt$ and $\vpt$ are temperature softened student and teacher's predictions. Logits gradient for $\KD$ is given by:
\begin{equation}\label{eq:dist_grad}
\partial\gL^{\KD}/\partial z_i = (1-\lambda)(q_i - y_i) + (\lambda/T)(\evqt{i} - \evpt{i}).  \quad\text{Lets denote } \partial^{\KD}_i = \partial\gL^{\KD}/\partial z_i.
\end{equation}
\citet{yuan2019revisit} established the connection between $\KD$ and LS: In terms of gradient propagation, $\KD$ is equivalent to LS, when $T=1$, and teacher's probability distribution $\vp$ follows a uniform distribution, i.e., a Bayesian prior of the universe.
In other words, we can view $\KD$ as an adaptive version of label smoothing, suggesting it should inherit most of the regularization benefits from label smoothing, such as model regularization and better calibration, not being over-confident~\citep{muller2019does}. %

In the next two subsections, we analyze the unique characteristics of real teacher's distribution %
over uniform distribution, and demonstrate how they could potentially facilitate student model's training.

\subsection{Domain knowledge -- teacher injects class relationships prior
}\label{sec:analysis_class_rel}

KD leverages class relationships as captured by the teacher's probability distribution $\vp$ over the \emph{incorrect classes}. As argued by~\citet{hinton2015distilling} on MNIST dataset, model assigns relatively high probability for class `7', when the ground-truth class is `2'. In this section, we first confirm their hypothesis using empirical studies. Then, we provide new insights to explain how the teacher informs the class relationships to its student at optimality, and improves model quality. %

To illustrate that the teacher's distribution $\vp$ captures class relationships, we train ResNet-56 on CIFAR-100 dataset. CIFAR-100 contains 100 classes over 20 super-classes, with each super-class containing 5 sub-classes. Figures~\ref{fig:heatmap-k=100_t=3} and~\ref{fig:heatmap-k=100_t=10} show the heatmap for Pearson correlation coefficient on teacher's distribution $\vp$ at different temperatures. We sort the class indexes to ensure that the 5 classes from the same super-class appear next to each other. With a lower temperature in Figure~\ref{fig:heatmap-k=100_t=3}, there's no pattern on the heatmap showing class relationships. But as we increase the temperature in Figure~\ref{fig:heatmap-k=100_t=10}, classes within the same super-class clearly have a high correlation with each other, as seen in the block diagonal structure. This observation verifies that teacher's distribution $\vp$ indeed reveals class relationships, with proper tuning on the $\softmax$ temperature. %

In this work, we found that the teacher's predictions on incorrect classes also provides a prior for student model training. Before diving into the details, we recall the case of label smoothing~\citep{szegedy2016rethinking}:
\begin{itemize}
\itemsep0em
    \item From an optimization point of view,~\citet{he2019bag} showed that there is an optimal constant margin $\log(K(1-\epsilon)/\epsilon + 1)$, between the logit of the ground-truth $\evz_t$, and all other logits $\evz_{-t}$, using a label smoothing factor of $\epsilon$. For fixed number of classes $K$, the margin is a monotonically decreasing function of $\epsilon$.
    \item From geometry perspective, \citet{muller2019does} showed the logit $\evz_k = \vh^\top \vw_k$ for any class $k$ is a measure of squared Euclidean distance $\|\vh - \vw_k\|^2$ on latent space between the activations of the penultimate layer\footnote{Here $\vh$ can be concatenated with a ``1'' to account for the bias.} $\vh$, and weights $\vw_k$ for class $k$ in the last logit layer. 
\end{itemize}
Above findings suggest that label smoothing encourages $\|\vh - \vw_t\|^2 \ge \|\vh - \vw_{-t}\|^2$ and pushes all the other incorrect classes equally apart.
Following a similar proof technique, we extend to $\KD$:
\begin{prop}\label{prop:distance}
With $\KD$, the optimal solution of student's final logit layer weights $\{\vw^*_k,~\forall k\in[K]\}$ enforces different inter-class distances based on teacher's probability distribution $\vp$: 
$$\|\vh - \vw^*_i\|^2 < \|\vh - \vw^*_j\|^2~~~\text{iff}~~~\evp_i > \evp_j,~\forall i,j \in [K] \backslash t,$$
where $\vh$ is the activations of the penultimate layer. (See proof in Suppl. Section~\ref{sec:si-proof})
\end{prop}
From Figure~\ref{fig:heatmap-k=100_t=10}, teacher assigns higher probability to the classes within the same super-class, and hence $\KD$ encourages hierarchical clustering of logit layer weights based on the class relationships. 

\begin{figure*}[t]
    \centering
    \begin{subfigure}[b]{0.22\textwidth}
            \includegraphics[width=\textwidth]{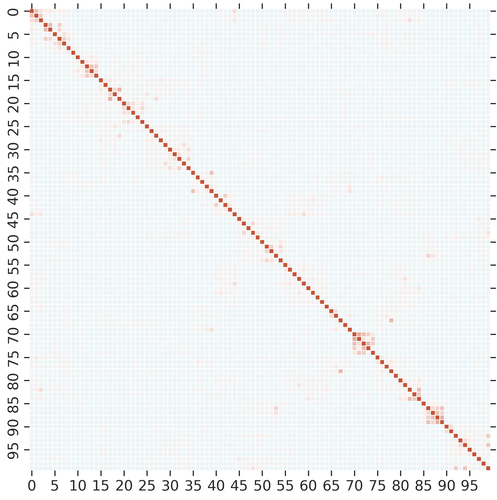}
            \subcaption{}
            \label{fig:heatmap-k=100_t=3}
    \end{subfigure}
    ~
    \begin{subfigure}[b]{0.22\textwidth}
            \includegraphics[width=\textwidth]{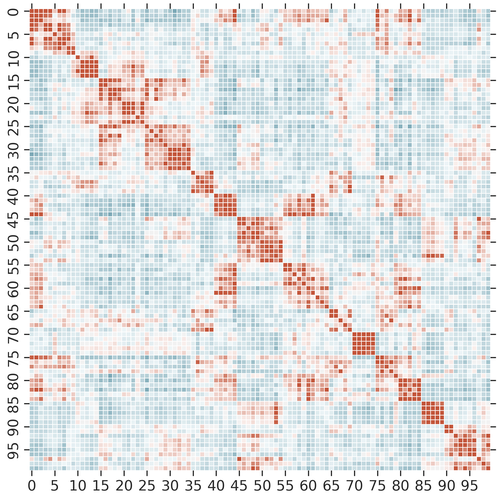}
            \subcaption{}
            \label{fig:heatmap-k=100_t=10}
    \end{subfigure}
    ~
    \begin{subfigure}[b]{0.22\textwidth}
            \includegraphics[width=\textwidth]{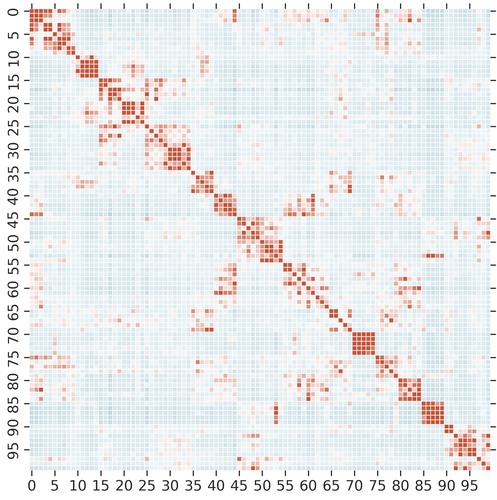}
            \subcaption{}
            \label{fig:heatmap-k=10-t=10}
    \end{subfigure}
    ~
    \begin{subfigure}[b]{0.26\textwidth}
            \includegraphics[width=\textwidth]{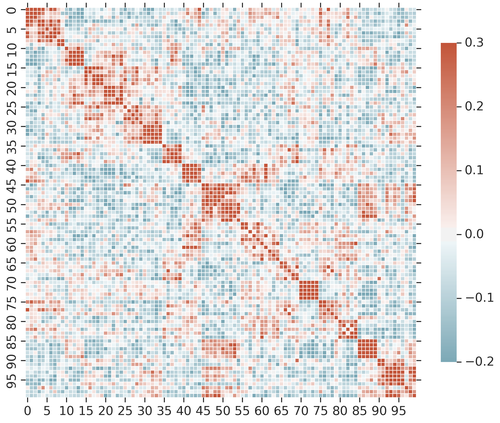}
            \subcaption{}
            \label{fig:heatmap-cosine-sims}
    \end{subfigure}
    \caption{\small Using 10K samples from CIFAR-100 for ResNet-56, we plot Pearson correlations of output probability $\vp$ with varying $\softmax$ temperature (a) $T=3$, (b) $T=10$, and (c) $T=10$ where only top-10 largest values in $\vp$ are preserved. In (d), we show cosine similarities computed from the weights of the final logits layer. Since classes within the super-class are grouped together, we see a block diagonal class correlation structure. %
    }
\end{figure*}
\subsection{Instance specific knowledge -- teacher rescales gradients based on event difficulty}\label{sec:analysis_re-weight}
Another important characteristic of the teacher distribution $\vp$ is that the prediction (confidence) $\evp_t$ on the \emph{ground-truth class} is different across instances. Comparing ratio of gradients (\cref{eq:grad,eq:dist_grad}):
\begin{equation}\label{eq:grad_scale}
\begin{aligned}
\omega_i = \frac{\partial^{\KD}_i}{\partial_i} = (1-\lambda) + \frac{\lambda}{T}\left( \frac{\evqt{i} - \evpt{i}}{\evq_i - \evy_i}\right),
\end{aligned}
\end{equation}
we find that $\KD$ performs gradient rescaling in the logits space based on teacher model's prediction confidence $\evp_t$ on the ground-truth class $t$. The gradient rescaling factor $\omega_i$ is larger on average, when teacher is more confident on making the right prediction. More specifically, we state the following:

\begin{prop}[\textbf{Gradient Rescaling}]\label{prop:conf}
Given any example $(\vx, \vy) \in \gX \times \gY$, let $\evpt{t} = \evqt{t} + \evct{t} + \eta$, where $\evct{t} > 0$ is teacher's relative prediction confidence on the ground-truth class $t\in [K]$ and $\eta$ is a zero-mean random noise. Then the logit's gradient rescaling factor by applying $\KD$ is given by:
$$\E_\eta\left[\frac{\partial^{KD}_t}{\partial_t}\right] = \E_\eta\left[\frac{\sum_{i\in[K]\backslash t}\partial^{KD}_i}{\sum_{i\in[K]\backslash t}\partial_i}\right] = (1 - \lambda) + \frac{\lambda}{T}\left(\frac{\evct{t}}{1 - \evq_t}\right).$$
\end{prop}
See proof in Suppl. Section~\ref{sec:si-proof}. At a given snapshot during training, we could assume ${\evct{t}}$ to be a constant for all examples. Then for any pairs of examples $(\vx, \vy)$, $(\vx', \vy') \in \gX \times \gY$, if the teacher is more confident on one of them, i.e., $\evp > \evp'$, then the average $\omega$ for all classes will be greater than $\omega'$.%

To validate our claim, in Figure~\ref{fig:p_t-w_t}, we plot the relationship between $\omega_t$ and $\evp_t$ at the end of training. On CIFAR-100~\citep{krizhevsky2009learning}, we use ResNet~\citep{he2016deep} with depth 20 as the student model, and depth 56 as the teacher (see Suppl. Section~\ref{sec:si-detail} for more details). The plot shows a clear positive correlation between the two. Notably, \emph{the correlation will be even stronger when closer to the beginning of training}.

In~\citep{furlanello2018born}, the authors conjecture that per-example weight is associated with the largest value in $\vp$.
In Proposition~\ref{prop:conf}, we show an alternative gradient rescaling effect. It's important to distinguish that the weight is associated with teacher's confidence on ground-truth, instead of the largest value.
Once the teacher makes a wrong prediction, using the largest value would yield contradictory result. It is also trivial to show that when we have two classes, $\omega_{i \neq t} = \omega_t$, the primary effect of $\KD$ is gradient rescaling. So we can regard the use of $\KD$ on binary classification~\citep{Anil2018Large} as taking the binary log-loss, and multiply with the weight $\omega_t$.

\begin{wrapfigure}{r}{0.4\textwidth}
    \centering
    \includegraphics[width=0.4\textwidth]{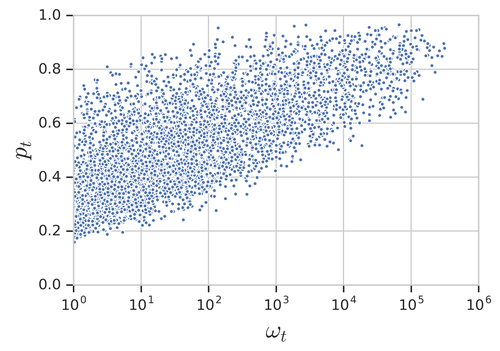}
    \vspace{-2em}
    \caption{\small Applying $\KD$ for ResNet-20 student model with ResNet-56 as the teacher on CIFAR-100, we plot $\evp_t$ vs. $\omega_t$ (in log scale) with 10K samples at the end of training.
    }
    \label{fig:p_t-w_t}
\end{wrapfigure}
Altogether, we show knowledge distillation has an effect of gradient rescaling with its factor associated with the teacher's prediction on ground-truth. Weight will be higher when $\evp_t$ is larger. Alternatively, this suggests that $\KD$ would magnify the gradients (w.r.t. logits), for training examples that are considered easier from teacher's perspective, and vice versa; which has a similar flavor to Curriculum Learning. \citet{bengio2009curriculum} suggested that this may speedup training convergence and helps optimization to reach a better local minima. 
This is also related to~\cite{roux2016tighter}, which shows that re-weighting examples during training using model's prediction confidence leads to a tighter bound on the classification error and leads to better generalization.

\subsection{Summary on primary effects of KD}\label{sec:summary_effects}
We conclude the section by summarizing KD's primary effects for classification task from three hierarchical levels of knowledge.
From the most general perspective, KD brings a regularization effect by introducing smoothened teacher distribution (i.e., a Bayesian prior of the universe). %
Then from domain knowledge, teacher's probability mass on the \emph{incorrect classes} reflect class relationships, therefore providing more guidance to the student. We showed the guidance is through influencing student's final logit layer geometry:
instead of pushing penultimate layer activations equally apart from the incorrect classes' weights as in label smoothing, $\KD$, as an adaptive label smoothing, encourages differences in inter-class distances.
Finally, for instance specific knowledge, the teacher model rescales student models's gradients with its measurement on event difficulty, i.e., confidence on the \emph{ground-truth class}.
As a result, all the three levels of knowledge complement each other, which could potentially facilitate student model's training process and further improve model generalization.

\section{Isolating Effects by Partial Knowledge Distillation Methods} \label{sec:pseudo_kd} To further dissect the different effects of $\KD$, in this section, we synthesize hand-crafted teacher distributions, denoted by $\vrho$.
Each synthetic teacher distribution $\vrho$ contains partial information from the real teacher's distribution $\vp$, enabling us to isolate and study the effects of $\KD$ (namely, gradient rescaling and prior on optimal geometry).
We propose $\kdpt$ and $\kdsim$
-- former only incorporates the gradient rescaling effect and excludes class relationship information, and the latter only incorporates class relationships but not gradient rescaling. We then try to combine the two effects together in an attempt to approximate the performance of vanilla $\KD$.

\textbf{Examine gradient rescaling effect by KD-pt.}
As discussed, label smoothing neither has gradient rescaling effect, nor information about class relationships, due to its uniform teacher distribution. However, if we borrow $\evp_t$ (prediction on ground truth class $t\in[K]$) from the real teacher's probability distribution $\vp$, we can synthesize a partial teacher distribution that is able to incorporate gradient rescaling effect. More specifically, we craft teacher's probability distribution $\vrho^{\text{pt}}$ as follows:
$\rho_i^{\textrm{pt}} = \evp_t ~\text{if } i = t, (1 - \evp_t)/(K-1)~\text{otherwise}.$
From Proposition~\ref{prop:conf}, it is trivial to see that $\kdpt$ is capable of rescaling gradients for different examples. However, it does not capture class relationships.%

\textbf{Examine optimal geometry prior by KD-sim.}
Following a similar methodology, we synthesize a teacher distribution that only captures class relationships, and ignores gradient rescaling.
To achieve this, we use the weights of the last logit layer $\mW \in \sR^{K \times d}$ from the teacher model to obtain class relationships. We believe the teacher, due to its larger capacity is able to encode class semantics in the weights of the last logit layer.
Thus, we create a distribution $\vrho^{\text{sim}}$ as the $\softmax$ over cosine similarity\footnote{In practice, 
besides tuning the temperature of the $\softmax$, one could also raise the similarities to a power $<1$ to amplify the resolution of cosine similarities. Please refer to Section~\ref{sec:si-detail} in Suppl. for more details.} of the weights: $\vrho^{\text{sim}} = \softmax(\hat{\vw}_t \Hat{\mW}^\top)$, where $\Hat{\mW}\in \sR^{K\times d}$ is the $\ell^2$-normalized logit layer weights, and $\hat{\vw}_t = \vw_t / \|\vw_t\|$ is the $t$-th row of $\Hat{\mW}$ corresponding to the ground truth. 
Though other distance metrics could also be used as a measure of class similarity, we leave the discussion of analysing the different choices as future work.
To verify our assumption, we check the heatmap of cosine similarities in Figure~\ref{fig:heatmap-cosine-sims}, which clearly shows a similar pattern as the Pearson correlation of the teacher's distribution $\vp$ in Figure~\ref{fig:heatmap-k=100_t=10}.
From Propositions~\ref{prop:conf} and \ref{prop:distance}, our proposed method, though simple and straightforward, can preserve class relationships only, and therefore achieve our purpose.

Note that $\kdsim$ doesn't require a prior knowledge of class hierarchy, but if available (as in CIFAR-100), we could also synthesize a teacher's probability distribution apriori.
In Suppl. Section C, we synthesize $\vrho$ by setting different values for (1) ground-truth class $t$, (2) classes within the same super-class of $t$, and (3) other incorrect classes. The quality of the resulting method is slightly poor compared to $\kdsim$, but still improves student model's generalization.

\textbf{Compounded effects.}
We also explore a combination of the above proposed orthogonal partial $\KD$ techniques to validate if the resulting method can approximate vanilla $\KD$.
We study a simple linear combination of synthetic teacher's probability distribution, that is, $(1-\alpha) \vrho^{\text{pt}} + \alpha \vrho^{\text{sim}}$ and name the method $\kdptsim$. %
It is easy to verify that this compounded method smoothens out the label distribution, rescales gradients, and also injects optimal prior geometry through class relationships.

\section{Empirical Studies} \label{sec:application}  In this section, we evaluate the effectiveness of our proposed partial-distillation methods, to better understand how much each of these effects benefits the student model, and how the improvements are associated with the dataset properties. With our understandings, we propose a simple way to improve distillation quality and we diagnose the recent failures of KD.

\subsection{How does class correlations influence distillation?}\label{sec:exp_synthetic}
Performance of $\KD$ is dependent on the dataset properties. A natural question is -- \emph{Does $\KD$ perform only gradient rescaling when all the classes are uncorrelated to each other?} We showed this to be true for binary classification (Section~\ref{sec:analysis_re-weight}). To answer the same for multi-class classification task, we generate synthetic dataset, where we can control the class similarities within the same super-class.

\textbf{Setup.}
Inspired by~\citep{ma2018modeling}, we synthesize a classification dataset with $K$ classes and $C$ super-classes, such that each super-class has $K/C$ classes, and each class will be assigned with a carefully generated basis vector, so that we could control the class correlations within the same super-class. Also, data points can be generated to be \emph{linearly non-separable} to control for task difficulty.
See Suppl. Section~\ref{sec:si-synthetic} for more details. %
In our experiments, we set input dimension $d=500$ with $K=50$ and $C=10$. We use $|\train|=500k$ data-points for training, and $|\valid|=50k$ for validation. We use a simple 2-layer fully-connected neural network with $\tanh$ activation, and hidden layer dimensions $64$ for the student, and $128$ for the teacher. 
By injecting non-linearities when generating the synthetic data, we are able to control the task difficulty trade-off (i.e., not too easy, but hard enough to have a large margin between the two models for $\KD$).
Figure~\ref{fig:synthetic-sim-m} in Suppl. shows a visualization of a toy dataset.

\textbf{Results and analysis.}
Table~\ref{tb:exp_synthetic_data} shows the classification accuracy on the validation set when varying class similarity within each super-class (denote as $\tau$).
We notice a large margin between the teacher and student, and
Knowledge Distillation ($\KD$) benefits the student significantly.
Interestingly, when all classes are uncorrelated ($\tau=0.0$), we notice $\kdpt$ even outperform $\KD$, verifying our claim of gradient scaling effect of $\KD$.
When increasing $\tau$, we see a significant improvement in performance of $\kdsim$, suggesting that the injected prior knowledge of class relationships can also aid student model in generalization.
Note that for this task, the data points that are close to the decision boundary are harder to classify, and can be regarded as difficult examples.
It is worth mentioning that the performance of both the student and teacher drastically drop when having larger $\tau$, as the classes within the same super-class will be too similar and difficult to distinguish. 
\begin{table}
\small
\parbox{.45\linewidth}{
\centering
\caption{\small Accuracy (\%) on synthetic dataset with different class similarities within each super-class.}
\begin{tabular}{l|ccc}
\toprule
\textbf{Method} & \textbf{$\tau=0.0$} & \textbf{$\tau=0.3$} & \textbf{$\tau=0.4$}\\
\midrule
Teacher & 54.82 & 60.60 & 61.53\\
Student & 38.45 & 42.81 & 47.10\\
\midrule
KD & 55.97 & 56.95 & \underline{57.70}\\
KD-pt & \underline{57.05} & \underline{56.99} & 56.25\\
KD-sim & 51.90 & 53.69 & 57.20\\
\bottomrule
\end{tabular}
\label{tb:exp_synthetic_data}

\caption{\small Mean and Std. for top-1 accuracy (\%) over 4 individual runs. Best $k$ for $\kdtopk$ is 25 and 500 for CIFAR-100 and ImageNet, resp.}
\begin{tabular}{l|cc}
\toprule
\textbf{Method} & \textbf{CIFAR-100} & \textbf{ImageNet}\\
\midrule
Teacher & 75.68~{\scriptsize $\pm$~0.42} & 77.98~{\scriptsize $\pm$~0.12} \\
Student & 72.51~{\scriptsize $\pm$~0.27} & 76.34~{\scriptsize $\pm$~0.11}\\
\midrule
LS & 73.87~{\scriptsize $\pm$~0.16}& 76.83~{\scriptsize $\pm$~0.07}\\
KD & 75.94~{\scriptsize $\pm$~0.26}& 77.49~{\scriptsize $\pm$~0.07}\\
\midrule
KD-pt & 75.08~{\scriptsize $\pm$~0.16} & 77.00~{\scriptsize $\pm$~0.08}\\
KD-sim & 74.30~{\scriptsize $\pm$~0.17}& 76.95~{\scriptsize $\pm$~0.07}\\
KD-pt+sim & 75.24~{\scriptsize $\pm$~0.17}& 77.17~{\scriptsize $\pm$~0.08}\\
\midrule
KD-topk & \underline{76.17}~{\scriptsize $\pm$~0.25} & \underline{77.85}~{\scriptsize $\pm$~0.03}\\
\bottomrule

\end{tabular}
\label{tb:exp_real_data}
}
\hfill
\parbox{.45\linewidth}{
\centering

\caption{\small Best validation and test Perplexity (lower is better) over 4 individual runs on PTB language modeling. Best $k$ value for $\kdtopk$ is 100.}
\begin{tabular}{l|ccc}
\toprule
\textbf{Method} & \textbf{\#Params} & \textbf{Validation} & \textbf{Test}\\
\midrule
Teacher & 24.2M & 60.90 & 58.58\\
Student & 9.1M & 64.17 & 61.55\\
\midrule
KD & 9.1M & 64.04 & 61.33\\
KD-topk & 9.1M & \underline{63.59} & \underline{60.85}\\
\bottomrule
\end{tabular}
\label{tb:exp_real_data_lm}

\includegraphics[width=0.4\columnwidth]{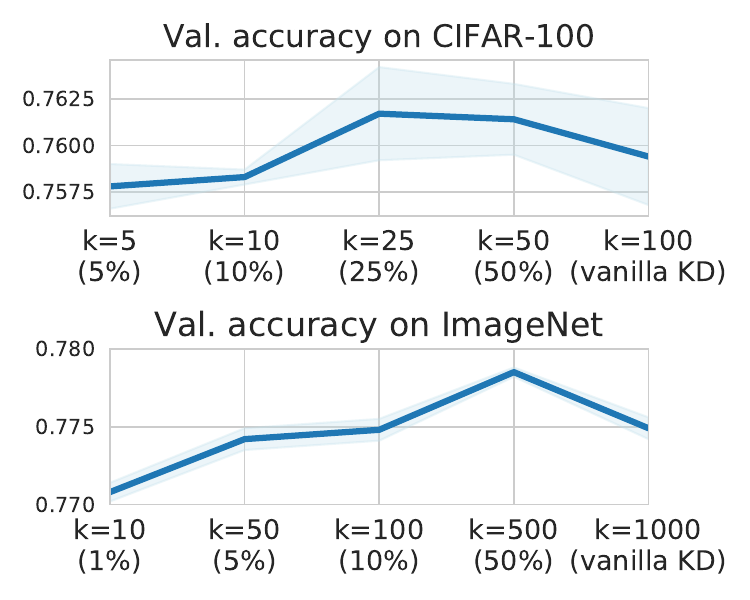}
\label{fig:exp_topk}
\vspace{-1em}
\captionof{figure}{\small Top-1 accuracy vs. $k$ for $\kdtopk$ from 4 individual runs on CIFAR-100 and ImageNet.}
}
\end{table}

\subsection{How effective are the partial-distillation methods?}
We next use two popular image classification datasets -- CIFAR-100~\citep{krizhevsky2009learning} and ImageNet~\citep{russakovsky2015imagenet} to analyze the quality of our proposed partial-distillation methods, and also to verify if we could approximate the performance of $\KD$ by compounding effects.

\textbf{Setup.}
On CIFAR-100 we use ResNet-20 as the student, and ResNet-56 as the teacher. On ImageNet with 1000 classes, we use ResNet-50 as the student, and ResNet-152 as the teacher.
For more details, please refer to Section~\ref{sec:si-detail} in Suppl.
Note that instead of using different model families as in~\citep{furlanello2018born,yuan2019revisit}, we use the same model architecture (i.e., ResNet) with different depths for the student and teacher to isolate any unknown effects introduced by model family discrepancy.

\textbf{Results and analysis.}
Table~\ref{tb:exp_real_data} shows the overall performance with the best hyper-parameters for each of the methods. On both datasets, teacher model is much better than the student, and label smoothing ($\LS$) improves student model's generalization. $\KD$ can further boost student model's quality by a large margin, especially on CIFAR-100, where $\KD$ even outperforms the teacher. We try to uncover the different benefits from distillation using partial-$\KD$ methods. Both $\kdpt$, and $\kdsim$ outperforms $\LS$; especially $\kdpt$ on CIFAR-100. This suggests that the different effects from $\KD$ benefits the student in different aspects depending on the dataset. 
Furthermore, by combining the two effects together in $\kdptsim$ (using $\alpha=0.5$), we see a further improvement in quality. 

\subsection{Regulated knowledge sharing improves distillation}
Following our understanding, any methods that can enhance and balance knowledge sharing at the three granular levels could potentially improve knowledge distillation.
Extending $\kdpt$, we take a step forward to use the top-$k$ largest values of teacher's probability $\vp$, and uniformly distributes the rest of the probability mass to the other classes, we name this method as $\kdtopk$.
For better intuition, from Figure~\ref{fig:heatmap-k=10-t=10} we observe that only preserving top-$10$ largest values could closely approximate the class correlations as in the full teacher's distribution $\vp$, and is also less noisy. 
This finding shows that only a few incorrect classes that are strongly correlated with the ground-truth class are useful for $\KD$ to boost the domain knowledge, and the probability mass on other classes are random noise (which is not negligible under high temperature $T$), and only \emph{has the effect of label smoothing in expectation}.
Furthermore, $\kdtopk$ can better utilize instance specific class relationships, since in $\kdsim$, all examples from the same class will have the same relationships to the other classes, which is restricted. 
For example, on MNIST dataset, only some versions of `2' looks similar to `7'.

Using the above intuition, we test $\kdtopk$ for image classification on CIFAR-100 and ImageNet, and language modeling on Penn Tree Bank (PTB) dataset. We apply state-of-the-art LSTM model~\citep{merity2017regularizing} with different capacities for the teacher and student. 
Details of PTB dataset and model specifications are in Section~\ref{sec:si-detail} of Suppl. For image classification, the performance of $\kdtopk$ is shown in the last row of Table~\ref{tb:exp_real_data}. We see that $\kdtopk$ outperforms $\KD$ on both datasets. For language modeling, the results are shown in Table~\ref{tb:exp_real_data_lm}, which suggests a similar trend for $\kdtopk$.
We plot the performance uplift of $\kdtopk$ along with $k$ in Figure~\ref{fig:exp_topk}. As shown, the best performance is achieved with a proper tuning of $k < K$, which captures class relationships and also reduces noise. The results align with our understanding of KD, and also suggest a way to achieve better distillation quality.

\subsection{Diagnosis of failure cases}
\label{sec:failure_cases}
\begin{table}
\small
\parbox{.4\linewidth}{
\centering
\caption{On CIFAR-100, Top-1 accuracy (\%) of various KD methods using teachers with or without label smoothing.
}
\begin{tabular}{l|cc}
\toprule
\multirow{2}{*}{\textbf{Method}} & \multicolumn{2}{c}{Teacher: ResNet-56}\\
& $\epsilon=0.0$ & $\epsilon=0.1$\\
\midrule
Teacher & 75.39 & 76.69~$\uparrow$\\
\midrule
KD & 76.00 & 75.02~$\downarrow$\\  
KD-pt & 74.81 & 74.13~$\downarrow$\\
KD-sim & 74.40 & 74.01~$\downarrow$\\
\bottomrule
\end{tabular}
\label{tb:label-smooth-failure}
}
\hfill
\parbox{.55\linewidth}{
\centering

\includegraphics[width=0.5\columnwidth]{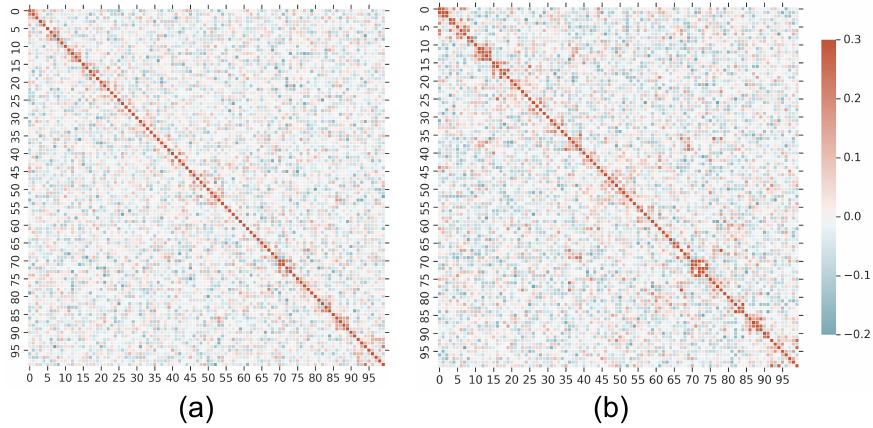}
\captionof{figure}{\small On CIFAR-100, for teacher with LS ($\epsilon=0.1$), we plot (a) Pearson correlations with $T=10$, and (b) cosine similarities computed from the weights of the final logits layer.}
\label{fig:heatmap-label-smooth-failure}
}
\end{table}

\begin{figure*}[t]
\centering
\vspace{-1.5em}
\includegraphics[scale=0.515]{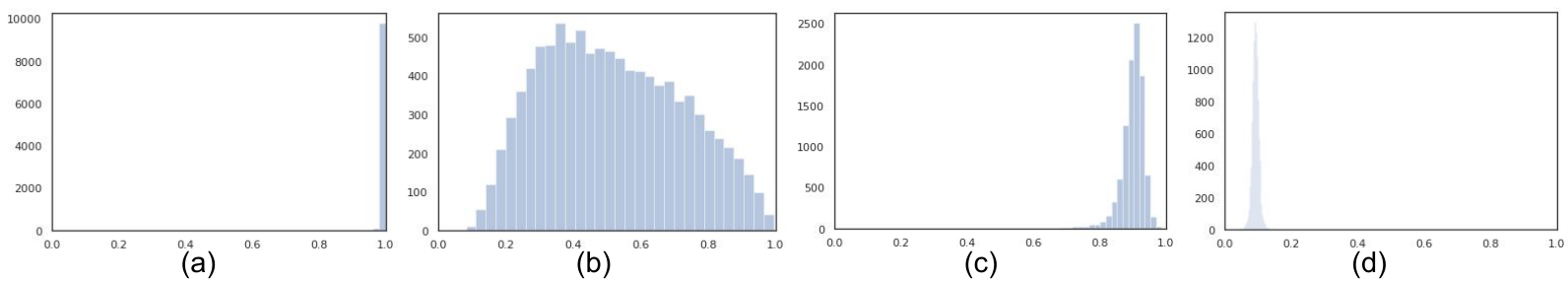}
\vspace{-0.5em}
\caption{\small On CIFAR-100, we plot the histogram of ResNet-56 teacher's confidence on ground-truth $\evp_t$, with different levels of label smoothing factor $\epsilon$: (a) $\epsilon=0.0$ and $T=1$; (b) $\epsilon=0.0$ and $T=5$; (c) $\epsilon=0.1$ and $T=1$ and (d) $\epsilon=0.1$ and $T=3$. The distribution of $\evp_t$ becomes skewed after enabling label smoothing.}
\label{fig:histogram-pt}
\vspace{-1em}
\end{figure*}

Having a good understanding of $\KD$ enables us to diagnose failure cases.
\citet{muller2019does} observed that although label smoothing (LS) improves teacher model's quality, it results in a worse student model when applying $\KD$.
Verified on CIFAR-100 in Table~\ref{tb:label-smooth-failure}, we found that the unfavorable distillation performance could be attributed to two factors -- Firstly, as argued by the~\citet{muller2019does} and illustrated in Figure~\ref{fig:heatmap-label-smooth-failure}, $\LS$ destroys class relationship information. Secondly, we found that the skewed teacher's prediction distribution on the ground-truth (see Figure~\ref{fig:histogram-pt}) also hinders the effectiveness of $\KD$, especially that of $\kdpt$, since gradient rescaling will be less effective. Results of $\kdsim$ and $\kdpt$ from last two columns of Table~\ref{tb:label-smooth-failure} verifies our hypothesis.

For another failure case, \citet{mirzadeh2019improved} showed that the `distilled' student model's quality gets worse as we continue to increase teacher model's capacity. Larger capacity teacher might overfit, and predict high (uniform) confidence on the ground truth on all the examples; and thereby hindering the effectiveness of gradient rescaling. Another explanation could be that there exists an optimal model capacity gap between the teacher and student, which could otherwise result in an inconsistency between teacher's prediction confidence on the ground-truth, and the desired example difficulty for the student.
Perhaps, an `easy' example for larger capacity teacher is overly difficult for the student. %

\section{Conclusion and Future Work}  \label{sec:conclusions} We provide novel techniques to better understand the mechanism of knowledge distillation~(KD).
Through systematic analyses, we uncover two key beneficial effects of $\KD$ over label smoothing. Firstly, supervision from teacher's prediction on the ground-truth rescales student's gradients for different training examples. Secondly, teacher's probability mass on the incorrect classes reveals class relationships by injecting prior knowledge of the optimal geometry of student's output layer.
These effects also explain why sometimes a better teacher may not be suitable for distillation, and self-distillation gives quality gains.
To have a closer look at these two effects, we proposed partial-distillation methods, and evaluated their performance on both synthetic and real-world datasets. Experimental results support our claims, help diagnose unpleasent results, and inspire ways to improve $\KD$.
In future work, we would like to extend our understanding of knowledge distillation under different data distributions, e.g., uniform vs long-tail distribution; and also consider the effect of noisy inputs and labels. We would like to also investigate other cheaper and effective ways of distillation \emph{e.g.,} looking into approximate versions of $\kdtopk$.

\section{Appendix} 
\subsection{Analyzing Mechanisms of Knowledge Distillation}
\label{sec:si-proof}

{\bf Proposition~\ref{prop:distance}.} (\emph{Paper}) \emph{With $\KD$, the optimal solution of student's final logit layer weights $\{\vw^*_k,~\forall k\in[K]\}$ enforces different inter-class distances based on teacher's probability distribution $\vp$: 
$$\|\vh - \vw^*_i\|^2 < \|\vh - \vw^*_j\|^2~~~\text{iff}~~~\evp_i > \evp_j,~\forall i,j \in [K] \backslash t,$$
where $\vh$ is the activations of the penultimate layer.}

\begin{proof}
At the optimal solution of the student, equating gradient in~\Eqref{eq:dist_grad} to $0$, we get:
\begin{equation}
(1-\lambda)(q^*_k - y_k) + \frac{\lambda}{T} (\evqt{k}^* - \evpt{k}) = 0 \implies (1-\lambda) q^*_k + \frac{\lambda}{T} \evqt{k}^* = (1-\lambda) y_k + \frac{\lambda}{T} \evpt{k}
\label{eq:opt_student_eq}
\end{equation}

Using a similar proof technique as~\citet{muller2019does}, $\|\vh - \vw^*_k\|^2 = \|\vh\|^2 + \|\vw^*_k\|^2 - 2 \vh^\top\vw^*_k$, where $\vh$ is the penultimate layer activations, and $\vw^*_k$ are the weights of the last logits layer for class $k \in [K]$. Note that $\|\vh\|^2$  is factored out when computing the $\softmax$, and $\|\vw^*_k\|^2$ is usually a (regularized) constant across all classes. Equating $\evz^*_k=\vh^\top\vw^*_k$, and using the property $\softmax(\vz) = \softmax(\vz + c),~\forall c\in\sR$, we get:
$$\evq^*_k = \softmax(\evz^*_k) = \softmax(\vh^\top\vw^*_k) =\softmax\Big(-\frac{1}{2}\|\vh - \vw^*_k\|^2\Big)$$

Plugging the above in~\eqref{eq:opt_student_eq}, we get:
\begin{equation*}
(1-\lambda) \softmax\Big(-\frac{1}{2}\|\vh - \vw^*_k\|^2\Big) + \frac{\lambda}{T} \softmax\Big(-\frac{1}{2T}\|\vh - \vw^*_k\|^2\Big) = (1-\lambda) y_k + \frac{\lambda}{T} \evpt{k}
\end{equation*}
Note that $\softmax$ is a monotonically increasing function, and we can rewrite LHS as:
\begin{equation*}
g\Big(-\frac{1}{2}\|\vh - \vw^*_k\|^2; \lambda, T\Big) = (1-\lambda) y_k + \frac{\lambda}{T} \evpt{k},
\end{equation*}
where $g(x; \lambda, T): \mathbb{R} \to \mathbb{R}$ is a monotonically increasing function, parameterized by $\lambda$ and $T$. Now for the incorrect classes, equating $y_k = 0$, and noting that $\softmax$ temperature scaling preserves relative ordering of teacher's probabilities $\evpt{k}$ proves the claim.
\end{proof}

{\bf Proposition~\ref{prop:conf}.}~[\textbf{Gradient Rescaling}] (\emph{Paper})
\emph{Given any example $(\vx, \vy) \in \gX \times \gY$, let $\evpt{t} = \evqt{t} + \evct{t} + \eta$, where $\evct{t} > 0$ is teacher's relative prediction confidence on the ground-truth class $t\in [K]$ and $\eta$ is a zero-mean random noise. Then the logit's gradient rescaling factor by applying $\KD$ is given by}:
$$\E_\eta\left[\frac{\partial^{KD}_t}{\partial_t}\right] = \E_\eta\left[\frac{\sum_{i\in[K]\backslash t}\partial^{KD}_i}{\sum_{i\in[K]\backslash t}\partial_i}\right] = (1 - \lambda) + \frac{\lambda}{T}\left(\frac{\evct{t}}{1 - \evq_t}\right).$$

\begin{proof}
We first consider the ground-truth class $t\in[K]$. Using $y_t=1$, $\evpt{t} = \evqt{t} + \evct{t} + \eta$ and $\E[\eta] = 0$ in \eqref{eq:grad_scale}, we get:
\begin{align*}
\E_{\eta}\left[\partial^{KD}_t/\partial_t\right] &= (1-\lambda) + \frac{\lambda}{T}\left(\frac{\evct{t}}{1 - \evq_t}\right)
\end{align*}

Now, sum of the incorrect class gradients is given by:%

\begin{align*}
\sum_{i \in[K]\backslash t} \partial^{\KD}_i &= \sum_{i\in[K]\backslash t}\big[ (1-\lambda){\evq}_i + \frac{\lambda}{T}(\evqt{i} - \evpt{i})\big]\\
& = (1-\lambda)(1-{\evq}_t) + \frac{\lambda}{T}(\evpt{t} - \evqt{t}) = -\partial^{\KD}_t
\end{align*}

Penultimate equality follows from $\vq,~\vpt$ and $\vqt$ being probability masses. Similarly applies for $\partial_i$, and hence the proof.

\end{proof}

\subsection{Experimental details}\label{sec:si-detail}

\begin{figure*}[t]
    \centering
    \begin{subfigure}[b]{0.3\textwidth}
            \includegraphics[width=\textwidth]{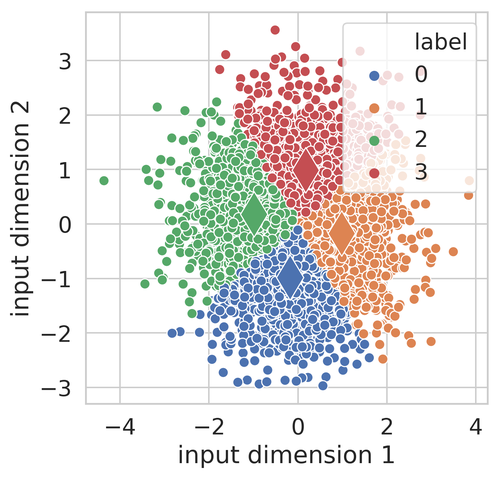}
            \subcaption{}
            \label{fig:synthetic-sim=0.0-m=0}
    \end{subfigure}
    ~
    \begin{subfigure}[b]{0.3\textwidth}
            \includegraphics[width=\textwidth]{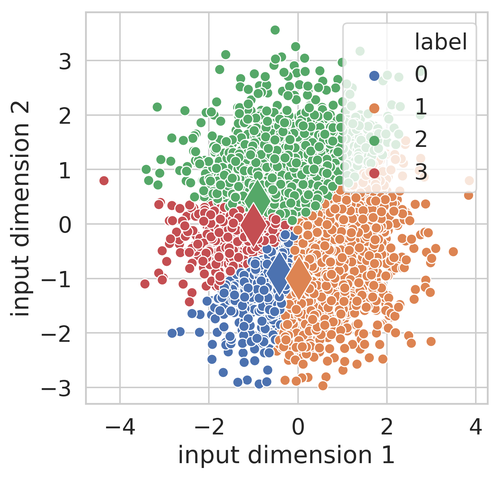}
            \subcaption{}
            \label{fig:synthetic-sim=0.9-m=0}
    \end{subfigure}
    ~
    \begin{subfigure}[b]{0.3\textwidth}
            \includegraphics[width=\textwidth]{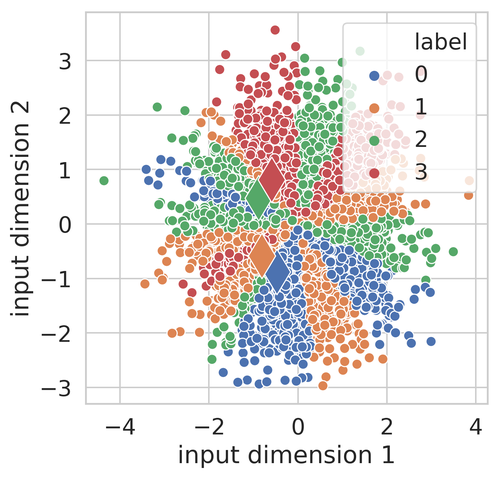}
            \subcaption{}
            \label{fig:synthetic-sim=0.9-m=2}
    \end{subfigure}
    
    \caption{Visualization of 5K synthetic data points (with input dimensionality $d=2$) on 2-D plane. We use $K=4$, $C=2$, means there are two super-classes, one associate with label \{0,1\} and the other one associate with label \{2,3\}. We vary $\tau$ and $M$ and produce 3 plots: (a) $\tau=0.0$, no sine function is used; (b) $\tau=0.9$, no sine function is used and (c) $\tau=0.9$, $M=2$.}
    \label{fig:synthetic-sim-m}
\end{figure*}

\paragraph{Implementation of $\KD$.} 
In practice, the gradients from the RHS of~\Eqref{eq:dist_grad} are much smaller compare to the gradients from LHS when temperature $T$ is large. Thus, it makes tuning the balancing hyper-parameter $\lambda$ become non-trivial. To mitigate this and make the gradients from two parts in the similar scale, we multiply $T^2$ to the RHS of~\Eqref{eq:dist_grad}, as suggested in~\citep{hinton2015distilling}.

\paragraph{Implementation of $\kdsim$.}
When synthesizing the teacher distribution for $\kdsim$, we use $\vrho^{\text{sim}} = \softmax(\hat{\vw}_t \Hat{\mW}^\top)$, where $\Hat{\mW}$ is the $l_2$-normalized logit layer weights and $\hat{\vw}_t$ is the $t$-th row of $\Hat{\mW}$. However, the cosine similarities computed for $\softmax$ are limited in the range of $[0,1]$. Therefore the resulting distribution is highly likely to be uniform. To mitigate this and bring more resolution to be cosine similarities, we use the following:
$$\vrho^{\text{sim}} = \softmax((\hat{\vw}_t \Hat{\mW}^\top)^{\alpha}/\beta).$$
Here $\alpha<1$ is a hyper-parameter to amplify the resolution of cosine similarities, $\beta$ is another hyper-parameter indicating the temperature for $\softmax$.

\begin{table}[t!]
    \begin{subtable}[h]{0.45\textwidth}
        \centering
        \begin{tabular}{l | l}
        \toprule
        {Method} & {Hyper-parameter setting}\\
        \hline
        LS & $\epsilon=0.3$ for any $\tau$.\\
        \hline
        \multirow{6}{*}{KD} & $\lambda=0.7, T=3$ when $\tau=0.0$\\
        & $\lambda=0.7, T=5$ when $\tau=0.1$\\
        & $\lambda=0.7, T=2$ when $\tau=0.2$\\
        & $\lambda=0.7, T=3$ when $\tau=0.3$\\
        & $\lambda=0.7, T=10$ when $\tau=0.4$\\
        & $\lambda=0.7, T=5$ when $\tau=0.5$;\\
        \hline
        \multirow{6}{*}{KD-pt} & $\lambda=0.7, T=3$ when $\tau=0.0$\\
        & $\lambda=0.7, T=5$ when $\tau=0.1$\\
        & $\lambda=0.7, T=2$ when $\tau=0.2$\\
        & $\lambda=0.7, T=3$ when $\tau=0.3$\\
        & $\lambda=0.7, T=10$ when $\tau=0.4$\\
        & $\lambda=0.7, T=5$ when $\tau=0.5$\\
        \hline
        KD-sim & $\lambda=0.7,\alpha=0.5,\beta=0.5$ for any $\tau$\\
        \bottomrule
       \end{tabular}
       \caption{Synthetic}
       \label{tb:hparam_synthetic}
    \end{subtable}
    \hfill
    \begin{subtable}[h]{0.45\textwidth}
        \centering
        \begin{tabular}{l | l}
        \toprule
        {Method} & {Hyper-parameter setting}\\
        \hline
        LS & $\epsilon=0.1$\\
        KD & $\lambda=0.3,T=5$\\
        KD-pt & $\lambda=0.3,T=5$\\
        KD-sim & $\lambda=0.3,\alpha=0.3,\beta=0.3$\\
        KD-topk & $k=25,\lambda=0.5,T=5$\\
        \bottomrule
       \end{tabular}
       \caption{CIFAR-100}
       \label{tb:hparam_cifar}
    \end{subtable}
    \hfill
    \begin{subtable}[h]{0.45\textwidth}
        \centering
        \begin{tabular}{l | l}
        \toprule
        {Method} & {Hyper-parameter setting}\\
        \hline
        LS & $\epsilon=0.1$\\
        KD & $\lambda=0.7,T=20$\\
        KD-pt & $\lambda=0.2,T=25$\\
        KD-sim & $\lambda=0.3,\alpha=0.5,\beta=0.3$\\
        KD-topk & $k=500,\lambda=0.5,T=3$\\
        \bottomrule
       \end{tabular}
       \caption{ImageNet}
       \label{tb:hparam_imagenet}
    \end{subtable}
    \hfill
    \begin{subtable}[h]{0.45\textwidth}
        \centering
        \begin{tabular}{l | l}
        \toprule
        {Method} & {Hyper-parameter setting}\\
        \hline
        KD & $\lambda=0.1,T=50$\\
        KD-topk & $k=100,\lambda=0.1,T=50$\\
        \bottomrule
       \end{tabular}
       \caption{Penn Tree Bank (PTB)}
       \label{tb:hparam_ptb}
    \end{subtable}
     \caption{Hyper-parameter settings for different methods on different datasets.}
     \label{tab:temps}
\end{table}

\paragraph{Synthetic dataset.}\label{sec:si-synthetic} 
For the synthetic dataset, we generate a single data-point as follows:
\begin{enumerate}
  \setlength{\itemsep}{1pt}
  \setlength{\parskip}{0pt}
  \setlength{\parsep}{0pt}
\item Randomly sample $C$ orthonormal basis vectors, denoted by $\vu_i \in\sR^d~\forall i\in[C]$.
    \item For each orthonormal basis $\vu_i$, we sample $(K/C-1)$ unit vectors $\vu_j\in\sR^d$ that are $\tau$ cosine similar to $\vu_i$.
    \item Randomly sample an input data point in $d$-dimensional feature space $\vx\sim\mathcal{N}_d(\mathbf{0}, \mathbf{I})$.
    \item Generate one-hot encoded label $\vy\in\gY$ with target: $t=\argmax_{k\in[K]} \big(\vu^\top_k \hat{\vx} + \sum_{m=1}^{M} \sin(\eva_m \vu^\top_k \hat{\vx} + \evb_m) \big)$, where $\hat{\vx}$ is the $l_2$-normalized $\vx$;  $\va,\vb\in\sR^M$ are arbitrary constants; and we refer to the controlled $\sin$ complexity term $M \in \sZ^+$ as \emph{task difficulty}.
\end{enumerate}

After producing basis vectors with procedure (1) and (2), we run procedure (3) and (4) for $|\train|$ times with fixed basis vectors to generate a synthetic dataset $\train=\{(\vx,\vy)\}$. By tuning the cosine similarity parameter $\tau$, we can control the classes correlations within the same super-class. Setting task-difficulty $M=0$ generates a linearly separable dataset, and $M>0$ generates more non-linearities by the $\sin$ function (see Figure~\ref{fig:synthetic-sim-m} in Suppl. for visualization on a toy example).

Following the procedure showed above, we get a toy synthetic dataset where we only have input dimensionality $d=2$ with $K=4$ classes and $C=2$ super-classes. Figure~\ref{fig:synthetic-sim-m} shows a series of scatter plots with different settings of class similarity $\tau$ and task difficulty $M$. This visualization gives a better understanding of the synthetic dataset and helps us imagine what it will look like in high-dimensional setting that used in our experiments. For the model used in our experiments, besides they are 2-layer network activated by $\tanh$, we use residual connection~\citep{he2016deep} and and batch normalization~\citep{ioffe2015batch} for each layer. Following~\citep{ranjan2017l2, zhang2018heated}, we found using $l_2$-normalized logits layer weight $\Hat{\mW}$ and penultimate layer $\hat{\vh}$ provides more stable results. The model is optmized by Adam~\citep{kingma2014adam} for a total of 3 million steps without weight decay and we report the best accuracy. Finally, Nvidia V100 GPU is used as the accelerator hardware.
Please refer to Table~\ref{tb:hparam_synthetic} for the best setting of hyper-parameters.

\paragraph{CIFAR-100 dataset.} 
CIFAR-100 is a relatively small dataset with low-resolution ($32\times32$) images, containing $50k$ training images and $10k$ validation images, covering $100$ classes and $20$ super-classes. It is a perfectly balanced dataset -- we have the same number of images per class (\emph{i.e.,} each class contains $500$ training set images) and $5$ classes per super-class.
To process the CIFAR-100 dataset, we use the official split from Tensorflow Dataset\footnote{\url{https://www.tensorflow.org/datasets/catalog/cifar100}}. Both data augmentation \footnote{\url{https://github.com/tensorflow/models/blob/master/research/resnet/cifar_input.py}}\footnote{We turn on the random brightness/saturation/constrast for better model performance.} for CIFAR-100 and the ResNet model\footnote{\url{https://github.com/tensorflow/models/blob/master/research/resnet/resnet_model.py}} are based on Tensorflow official implementations. Also, following the conventions, we train all models from scrach using Stochastic Gradient Descent (SGD) with a weight decay of 1e-3 and a Nesterov momentum of 0.9 for a total of 10K steps. The initial learning rate is 1e-1, it will become 1e-2 after 40K steps and become 1e-3 after 60K steps. We report the best accuracy for each model.
All experiments on CIFAR-100 are conducted by using Nvidia V100 GPU as the accelerator hardware.
Please refer to Table~\ref{tb:hparam_cifar} for the best setting of hyper-parameters.

\paragraph{ImageNet dataset.} 
ImageNet contains about $1.3$M training images and $50k$ test images, all of which are high-resolution ($224\times224$), covering $1000$ classes. The distribution over the classes is approximately uniform in the training set, and strictly uniform in the test set.
Our data preprocessing and model on ImageNet dataset are follow Tensorflow TPU official implementations\footnote{\url{https://github.com/tensorflow/tpu/tree/master/models/official/resnet}}. The Stochastic Gradient Descent (SGD) with a weight decay of 1e-4 and a Nesterov momentumof 0.9 is used. We train each model for 120 epochs, the mini-batch size is fixed to be 1024 and low precision (FP16) of model parameters is adopted. We didn't change the learning rate schedule scheme from the original implementation.
Please refer to Table~\ref{tb:hparam_imagenet} for the best setting of hyper-parameters. We used TPU-v3 as the accelerator hardware.

\paragraph{Penn Tree Bank dataset.}
We use Penn Tree Bank (PTB) dataset for word-level language modeling task using the standard train/validation/test split by~\citep{mikolov2010recurrent}. The vocabulary is
capped at 10K unique words. We consider the state-of-the-art LSTM model called AWD-LSTM proposed by~\citet{merity2017regularizing}. The model used several regularization tricks on top of a 3-layer LSTM, including DropConnect, embedding dropout, tied weight, etc. We use different capacity (indicated by hidden size and embedding size) as our Teacher and Student. To be specific, Teacher has a hidden size of 1150 and an embedding size of 400, while Student has a smaller hidden size of 600 and a smaller embedding size of 300. We follow the official implementation\footnote{\url{https://github.com/salesforce/awd-lstm-lm}} with simple changes for $\kdtopk$. Besides capacity, we keep the default hyper-parameter as in the official implementation to train our Teacher. However, when training smaller Student model, we follow another implementation\footnote{\url{https://github.com/zihangdai/mos}} to: (1) lower the learning rate to 0.2, (2) increase training epochs to 1000, (3) use 0.4 for embedding dropout rate and (4) use 0.225 for RNN layer dropout rate. We used Nvidia P100 GPU as the accelerator hardware.

\section{Additional Experiments}\label{sec:si-exp}

\begin{table}[t!]
\centering
\begin{tabular}{l|c}
\toprule
\textbf{Method} & \textbf{\% top-1 accuracy}\\
\hline
Student & 72.51\\
KD & 75.94\\
\hline
KD-rel & 74.14\\
KD-sim & 74.30\\
\hline
KD-pt+rel & 75.07\\
KD-pt+sim & 75.24\\
\bottomrule
\end{tabular}
\caption{Performance of $\kdrel$ on CIFAR-100. We report the mean result for 4 individual runs with different initializations. We use $\beta_1=0.6,\beta_2=\frac{0.1}{4},\beta_3=\frac{0.3}{95}$.}
\label{tb:exp_kdrel}
\end{table}
\paragraph{Examine optimal geometry prior effect with class hierarchy.}
In section~\ref{sec:pseudo_kd}, we mentioned the optimal geometry prior effects of $\KD$ can also be examined using existing class hierarchy. 
Suppose our data has a pre-defined class hierarchy (e.g., on CIFAR-100), we can also use it to examine the optimal geometry prior effects of $\KD$. To be specific, let $\sS_t \subset [K]\backslash t$ denote the other classes that share same parent of $t$.
We simply assign different probability masses to different types of classes:
\begin{equation}\label{eq:kdrel}
\rho_i^{\textrm{rel}} =  \begin{cases}
    \beta_1 & \text{if } i = t,\\
    \beta_2 & \text{if } i \in \sS_t,\\
    \beta_3 & \text{otherwise},
        \end{cases}
\end{equation}
where $\beta_1 > \beta_2 > \beta_3$ are a hyper-parameters we could search and optimize, and we name this method as $\kdrel$. 
As shown in Table~\ref{tb:exp_kdrel}, we found $\kdrel$ performs slightly worse  than $\kdsim$ on CIFAR-100. The trend is still hold when we compound each effect with $\kdpt$.

\bibliography{reference.bib}
\bibliographystyle{reference}
\end{document}